\documentclass[a4paper,11pt]{article}
\usepackage[top=4cm, bottom=4cm, left=3cm, right=3cm]{geometry}

\usepackage[utf8]{inputenc} 
\usepackage[T1]{fontenc}    
\usepackage{url}            
\usepackage{amsfonts}       
\usepackage{nicefrac}       
\usepackage{microtype}      

\usepackage{amssymb}
\usepackage{amsmath}
\usepackage{amsthm}

\usepackage{thmtools, thm-restate}
\usepackage{mathrsfs} 
\usepackage[usenames,dvipsnames]{pstricks}
\usepackage{euler}
\usepackage{euscript}
\usepackage{graphicx}
\usepackage{charter}
\usepackage{mdframed}
\usepackage{enumerate, subfigure, color}
\usepackage[colorlinks,hyperindex]{hyperref}
\usepackage{bibunits}

\usepackage{newfloat}

\usepackage{booktabs}
\usepackage{cite}

\usepackage{floatrow}

\usepackage[font=small,skip=0pt]{caption}

\newcommand{\X}{\mathcal{X}}
\newcommand{\Y}{\mathcal{Y}}

\newcommand{\E}{\mathcal{E}}

\newcommand{\B}{\mathcal{B}}

\renewcommand{\H}{\mathcal{H}}

\newcommand{\HY}{{\H_\Y}}

\newcommand{\trans}{^{\scriptscriptstyle \top}}

\newcommand{\R}{\mathbb{R}}

\newcommand{\N}{\mathbb{N}}

\renewcommand{\L}{\mathcal{L}}

\newcommand{\C}{\mathcal{C}}

\newcommand{\la}{\lambda}

\newcommand{\loss}{\bigtriangleup}

\newcommand{\refref}[2]{#1.~\ref{#2}}

\newcommand{\secref}[1]{\refref{Sec}{#1}}

\renewcommand{\eqref}[1]{Eq.~(\ref{#1})}

\newcommand{\ip}[2]{\langle#1,#2\rangle}

\newcommand{\minimize}[1]{\underset{#1}{\rm minimize}~~}

\newcommand{\eqals}[1]{\begin{align*}#1\end{align*}}

\newcommand{\eqal}[1]{\begin{align}#1\end{align}}

\renewcommand{\eqals}[1]{\eqal{#1}}

\newcommand{\argmin}{\operatornamewithlimits{argmin}}

\declaretheorem[name=Theorem,refname=Thm.]{theorem}
\declaretheorem[name=Lemma,sibling=theorem]{lemma}

\declaretheorem[name=Remark]{remark}

\declaretheorem[name=Definition,refname=Def.,sibling=theorem]{definition}

\declaretheorem[name=Example]{example}

\title{\sffamily\huge\bf Consistent Multitask Learning with Nonlinear Output Relations}

\author{\hspace*{-1.2em} Carlo Ciliberto $^{*,1}$ ~~~ Alessandro Rudi $^{*,2}$ ~~~ Lorenzo Rosasco $^{2,3}$ ~~~ Massimiliano Pontil $^{1,4}$ \\ {\small \hspace*{-3.6em} c.ciliberto@ucl.ac.uk ~~~~~~ ale\_rudi@mit.edu ~~~~~~~~~~~~  lrosasco@mit.edu ~~~~~~~~~~~ m.pontil@cs.ucl.ac.uk} \and {\small ${}^*$Equal contribution} }

\begin{document}

\maketitle

\begin{abstract}
  \noindent Key to multitask learning is exploiting relationships between different tasks
  to improve prediction performance. If the relations
  are linear, regularization approaches can be used successfully. However, in practice assuming the tasks to be linearly related might be restrictive, and allowing for nonlinear structures is a challenge. In this paper, we tackle this issue by
  casting the problem within the framework of structured prediction.
  Our main contribution is a novel 
  algorithm for learning multiple tasks which are related by a system of nonlinear equations that their joint outputs need to satisfy. We show that the algorithm is consistent and can be efficiently implemented.
  Experimental results show the potential of the proposed method.
\end{abstract}
\begin{bibunit}[unsrt]
\section{Introduction}
\footnotetext[1]{University College London, London, UK}\footnotetext[2]{Laboratory for Computational and Statistical Learning - Istituto Italiano di Tecnologia, Genova, Italy \& Massachusetts Institute of Technology, Cambridge, MA 02139, USA.}\footnotetext[3]{Universit\'a degli Studi di Genova, Genova, Italy}\footnotetext[4]{Computational Statistics and Machine Learning - Istituto Italiano di Tecnologia, Genova, Italy}
Improving the efficiency of learning from human supervision is one of the great challenges in machine learning. Multitask learning is one of the key approaches in this sense and it is based on the assumption that different learning problems (i.e. tasks) are often related, a property that can be exploited to reduce the amount of data needed to learn each individual tasks and in particular to learn efficiently novel tasks (a.k.a. transfer learning, learning to learn \cite{thrun2012}). Special cases of multitask learning include vector-valued regression and multi-category classification; applications are numerous, including classic ones in geophysics, recommender systems, co-kriging or collaborative filtering (see \cite{AlvLawRos12,argyriou2007,pan2010} and references therein). Diverse methods have been proposed to tackle this problem, for examples based on  kernel methods \cite{micchelli2004}, sparsity 
approaches \cite{argyriou2007} or neural networks \cite{bishop2006}. 
Furthermore, recent theoretical results allowed to quantify the benefits of multitask learning from a generalization point view when considering specific methods \cite{maurer2013,maurer2016}. 

A common challenge for multitask learning approaches is the problem of incorporating prior assumptions on the task relatedness in the learning process. This can be done implicitly, as in neural networks \cite{bishop2006}, or explicitly,  as done in  regularization methods by  designing suitable regularizers \cite{micchelli2004}. This latter approach is flexible  enough to incorporate different notions of tasks' relatedness expressed, for example, in terms of  clusters or a graph, see e.g. \cite{evgeniou05,jacob08}. Further, it can be extended to  {\em learn} the tasks' structures when they are unknown \cite{zhang10,dinuzzo11,ciliberto2015,jawanpuria2015}. However, regularization approaches are currently limited to imposing, or learning, tasks structures expressed 
by linear relations. For example an adjacency matrix in the case of graphs or a block matrix in the case of clusters.  Clearly while such a restriction might make the problem more amenable to statistical and computational analysis, in practice it might be  a severe limitation. 

Encoding and exploiting nonlinear task relatedness is the problem we consider in this paper. Previous literature on the topic is scarce. Neural networks naturally allow to learn with nonlinear relations, however it is unclear whether such relations can be imposed a-priori. As explained below, our problem is related to that of manifold valued regression \cite{steinke2009}. Close to our perspective is \cite{agarwal2010}, where a different approach is proposed, implicitly enforcing a nonlinear structure on the problem by requiring the parameters of each task predictors to lie
on a shared manifold in the hypotheses space.

Our main contribution is a novel method for learning multiple tasks which are {\em nonlinearly} related. We address this problem from the perspective of structured prediction (see \cite{bakir2007,tsochantaridis2005} and references therein) and build upon ideas recently proposed in \cite{ciliberto2016}. Specifically we look at multitask learning as the problem of learning a vector-valued function with image contained in a prescribed set, which models tasks' interactions. We also discuss how to deal with possible violations of such a constraint set.  To our knowledge this is the first work addressing the problem of imposing nonlinear output relations among multiple tasks.
\section{Problem Formulation}
\label{sec:problem-formulation}
Multitask learning (MTL) studies the problem of estimating multiple (real-valued) functions
\eqals{
   f_1, \dots, f_T:\X\to\R
}
from corresponding training sets $(x_{it},y_{it})_{i=1}^{n_t}$ with $x_{it}\in\X$ and $y_{it}\in\R$, for $t=1,\dots,T$.
%
The basic idea in MTL is to estimate $f_1, \dots, f_T$ jointly, rather than independently. The intuition is that if the different tasks are {\em related} this strategy can lead to a substantial decrease of sample complexity, that is the amount of data needed to achieve a given accuracy. The crucial question is then how to encode and exploit the relations among the tasks.

Previous work on MTL has mostly focused on studying the case where the tasks are linearly related (see \secref{sec:previous-work}). Indeed, this allows to capture a wide range of relevant situations and the resulting problem can be often cast as a convex optimization one, which can be solved efficiently. However, it is not hard to imagine situations where different tasks might be nonlinearly related. As a simple example consider the problem of learning two functions $f_1,f_2:[0,2\pi]\to\R$, with $f_1(x) = \cos(x)$ and $f_2(x) = \sin(x)$. Clearly the two tasks are strongly related one to the other (they need to satisfy $f_1(x)^2 + f_2(x)^2 - 1 = 0$ for all $x\in[0,2\pi]$) but such structure in nonlinear (here an equation of degree $2$). More realistic examples can be found for instance in the context of dynamical systems, such as the case of a robot manipulator. A prototypical learning problem (see e.g. \cite{rasmussen2006}) is to associate the current state of the system (position, velocity, acceleration) to a variety of measurements (e.g. torques) that are nonlinearly related one to the other by physical constraints (see e.g. \cite{sciavicco1996}).  

Following the intuition above, in this work we model tasks relations as a set of $P$ equations. 
Specifically we consider a {\em constraint function} $\gamma:\R^T\to\R^P$ and require that $\gamma\left(f_1(x),\dots,f_T(x)\right) = 0$ for all $x\in\X$.
When $\gamma$ is linear, the problem reverts to linear MTL and can be addressed 
via standard approaches (see \secref{sec:previous-work}). On the contrary, 
the nonlinear case becomes significantly more challenging and it is not clear how to address it in general. 
The starting point of our study is to consider the tasks predictors as a vector-valued function $f=(f_1,\dots,f_T):\X\to\R^T$ but then observe that $\gamma$ imposes constraints on its range. Specifically, in this work we restrict $f:\X\to\C$ to take values in the {\em constraint set} 
\begin{equation}\label{eq:constraint-set}
\C = \left\{y \in \R^T \ | \ \gamma(y) = 0 \right\} ~ \subseteq ~ \R^T
\end{equation}
and formulate the nonlinear multitask learning problem as that of finding a good approximation $\widehat f:\X\to\C$ to the solution of the multi-task {\em expected risk} minimization problem
\eqal{\label{eq:expected_risk}
    \minimize{f:\X\to\C} \E(f), \qquad \E(f) = \frac{1}{T}\sum_{t=1}^T\int_{\X\times\R} \ell(f_t(x),y)d\rho_t(x,y)
}
where $\ell:\R\times\R\to\R$ is a prescribed loss function measuring prediction errors for each individual task and, for 
every $t=1,\dots,T$, $\rho_t$ is the distribution on $\X\times\R$ from which the training points $(x_{it},y_{it})_{i=1}^{n_t}$ have been independently sampled.

Nonlinear MTL poses several challenges to standard machine learning approaches. Indeed, when $\C$ is a linear space (e.g. $\gamma$ is a linear map) the typical strategy to tackle problem (\ref{eq:expected_risk}) is to minimize the {\em empirical risk} $\frac{1}{T}\sum_{t=1}^T\frac{1}{n_t}\sum_{i=1}^{n_t} \ell(f_t(x_{it}),y_{it})$
over some suitable space of hypotheses $f:\X\to\C$ within which 
optimization can be performed efficiently. However, if $\C$ is a nonlinear subset of $\R^T$, it is not clear how to parametrize a ``good'' space of functions since most basic properties typically needed by optimization algorithms are immediately lost (e.g. $f_1,f_2:\X\to\C$ does not necessarily imply $f_1 + f_2:\X\to\C$). To address this issue, in this paper we adopt the {\em structured prediction} perspective proposed in \cite{ciliberto2016}, which we review in the following.
\subsection{Background: Structured Prediction and the SELF Framework}
\label{sec:self-overview}
The term structured prediction typically refers to supervised learning problems with discrete outputs, such as strings or graphs~\cite{bakir2007,tsochantaridis2005,nowozin2011}. The framework in \cite{ciliberto2016} strays from this perspective to consider a more general formulation, which is akin to that at \eqref{eq:expected_risk}, namely to learn   an estimator for
\eqal{\label{eq:expected_struct}
  \minimize{f:\X\to\C} \int_{\X\times\Y} \L(f(x),y) d\rho(x,y)
}
given a training set $(x_i,y_i)_{i=1}^n$ sampled iid from an unknown distribution $\rho$ on $\X\times\Y$, where $\L:\Y\times\Y\to\R$ is a loss function. The output sets $\C\subseteq\Y$ are not assumed to be linear spaces but can be either discrete (e.g. strings, graphs, etc.) or dense (e.g. manifolds, distributions, etc.) sets of ``structured'' objects. This generalization will be crucial to tackle the question of multitask learning with nonlinear output relations in \secref{sec:algorithm} since it allows to consider the case where $\C$ is a generic subset of $\Y = \R^T$. The analysis in \cite{ciliberto2016} hinges on the assumption that the loss $\L$ is ``bi-linearizable'', namely
\begin{definition}[{\bf SELF}]\label{asm:self}
Let $\Y$ be a compact set. A function $\ell:\Y\times \Y\to\R$ is a {\em Structure Encoding Loss Function} if there exists a continuous feature map $\psi:\Y\to\H$, with $\H$ a reproducing kernel Hilbert space on $\Y$ and a continuous linear operator $V:\H\to\H$ such that for all $y,y'\in \Y$
\begin{equation}\label{eq:self}
\ell(y,y') = \ip{\psi(y)}{V\psi(y')}_\H.
\end{equation}
\end{definition}
%
In the original work the SELF definition was dubbed ``loss trick'' as a parallel to the {\em kernel trick}~\cite{scholkopf2002}.  As we discuss in \secref{sec:theory}, most MTL loss functions indeed satisfy the SELF property. Under the SELF assumption, it can be shown that a solution $f^*:\X\to\C$ to \eqref{eq:expected_struct} must satisfy
\eqal{\label{eq:struct_solution}
  f^*(x) = \argmin_{c\in\C} ~ \ip{\psi(c)}{V ~ g^*(x)}_\H \qquad \textrm{with} \qquad g^*(x) = \int_\Y \psi(y) ~d\rho(y|x)
} 
for all $x\in\X$ (see \cite{ciliberto2016} or the Appendix). 
Since $g^*:\X\to\H$ is a function taking values in a linear space, we can apply standard vector-valued regression techniques to learn a $\widehat g:\X\to\H$ to approximate $g^*$ given $(x_i,\psi(y_i))_{i=1}^n$ and then obtain the estimator $\widehat f:\X\to\C$ as
\eqal{\label{eq:self_approx}
  \widehat f(x) = \argmin_{c\in\C}~ \ip{\psi(c)~ }{V~\widehat g(x)}_\H.
}
The intuition here is that if $\widehat g$ is close to $g^*$, so it will be $\widehat f$ to $f^*$ (see \secref{sec:theory} for a rigorous analysis of such relation). In particular, if $\widehat g$ is the {\em kernel ridge regression} estimator obtained by minimizing the empirical risk $\frac{1}{n}\sum_{i=1}^n \|g(x_i) - \psi(y_i)\|_\H^2$ (plus regularization),~\eqref{eq:self_approx} becomes
\eqal{\label{eq:loss-trick}
  \widehat f(x) = \argmin_{c\in\C}~ \sum_{i=1}^n \alpha_i(x)\L(c,y_i), \qquad \alpha(x) = (\alpha_1(x),\dots,\alpha_n(x))\trans = (K + n\lambda I)^{-1}K_x
}
since $\widehat g$ has form $\widehat g(x) = \sum_{i=1}^n \alpha_i(x)~\psi(y_i)$ and the loss function $\L$ is SELF. In the above formula $\lambda>0$ is a hyperparameter, $I\in\R^{n \times n}$ the identity matrix, $K\in\R^{n \times n}$ the kernel matrix with elements $K_{ij} = k(x_i,x_j)$, $K_x\in\R^n$ the vector with entries $(K_x)_i = k(x,x_i)$ and $k:\X\times\X\to\R$ a reproducing kernel on $\X$.
%
The SELF structured prediction approach is therefore conceptually divided into two distinct phases: a {\em learning} step, where the score functions $\alpha_i:\X\to\R$ are estimated, which consists in solving the kernel ridge regression in $\widehat g$, followed by a {\em prediction} step, where the vector $c\in\C$ minimizing the weighted sum in \eqref{eq:loss-trick} is identified.  
Interestingly, while the feature map $\psi$, the space $\H$ and the operator $V$ allow to derive the SELF estimator, 
{\em their knowledge is not needed to evaluate $\widehat f(x)$ in practice} since the optimization at \eqref{eq:loss-trick} depends exclusively on the loss $\L$ and the score functions $\alpha_i$.
%
\section{Structured Prediction for Nonlinear MTL}
\label{sec:algorithm}
In this section we present the main contribution of this work, namely the extension of the SELF framework to the MTL setting. Furthermore, we discuss how to cope with possible violations of the constraint set in practice. We begin our analysis by applying SELF to vector-valued regression.

\subsection{Nonlinear Vector-valued Regression}
Vector-valued regression (VVR) is a special instance of MTL where for each input, all output examples are available during training. In other words, the training sets can be combined into a single dataset $(x_i,y_i)_{i=1}^n$, with $y_i = (y_{i1},\dots,y_{it})^\top\in\R^T$. If we denote $\L:\R^T\times\R^T\to\R$ the separable loss $\L(y,y') = \frac{1}{T}\sum_{t=1} \ell(y_t,y_t')$, nonlinear VVR coincides with the structured prediction problem in \eqref{eq:expected_struct}. If $\L$ is SELF, we can therefore obtain an estimator according to \eqref{eq:loss-trick}.
\begin{example}[Nonlinear VVR with Square Loss]\label{ex:ls-and-C}
Let $\L(y,y') = \sum_{t=1}^T (y_t - y_t')^2$. Then, an estimator for nonlinear VVR can be obtained as $\widehat f:\X\to\C$ from \eqref{eq:loss-trick} and corresponds to the projection
\begin{equation}\label{eq:ls-and-C-solution}
\widehat f(x) = \argmin_{c\in\C}~ \left\|c - b(x)/a(x)\right\|^2 = \Pi_\C \left(b(x)/a(x)\right)
\end{equation}
with $a(x) = \sum_{i=1}^n \alpha_i(x)$ and $b(x) = \sum_{i=1}^n \alpha_i(x)~y_i$. Interestingly, $b(x) = \sum_{i=1}^n \alpha_i(x) y_i = Y^\top (K + n \lambda I)^{-1} K_x$ corresponds to the solution of the standard vector-valued kernel ridge regression {\em without} constraints (we denoted $Y\in\R^{n \times T}$ the matrix with rows $y_i^\top$). Therefore, nonlinear VVR consists in: $1)$ computing the {\em unconstrained} kernel ridge regression estimator $b(x)$, $2)$ normalizing it by $a(x)$ and $3)$ projecting it onto $\C$.
\end{example}
The example above shows that for specific loss functions the estimation of $\widehat f(x)$ can be significantly simplified. In general, such optimization will depend on the properties of the constraint set $\C$ (e.g. convex, connected, etc.) and the loss $\ell$ (e.g. convex, smooth, etc.). In practice, if $\C$ is a  discrete (or discretized) subset of $\R^T$, the computation can be performed efficiently via a nearest neighbor search (e.g. using {\em k-d trees} based approaches to speed up computations \cite{cormen2009}). If $\C$ is a manifold, recent {\em geometric optimization} methods \cite{sra2016} (e.g. SVRG \cite{zhang16}) can be applied to find critical points of \eqref{eq:loss-trick}. We conclude by discussing a connection between nonlinear VVR and manifold regression.
\begin{remark}[Connection to Manifold Regression]\label{rem:manifold_regression}
When $\C$ is a Riemannian manifold, the problem of learning $f:\X\to\C$ shares some similarities to the {\em manifold regression} setting studied in \cite{steinke2009} (see also \cite{steinke2010} and references therein). Manifold regression can be interpreted as a vector-valued learning setting where outputs are constrained to be in $\C\subseteq\R^T$ and prediction errors are measured according to the {\em geodesic distance}. However, note that the two problems are also significantly different since, in MTL, noise could make output examples $y_i$ lie close but not exactly on the constraint set $\C$ and moreover, the loss functions used in MTL typically measure errors independently for each task (as in \eqref{eq:expected_risk}, see also~\cite{micchelli2004}) and rarely coincide with a geodesic distance.
\end{remark}
\subsection{Nonlinear Multitask Learning}\label{sec:nonlinear-mtl}
Differently from nonlinear vector-valued regression, the SELF approach introduced in \secref{sec:self-overview} cannot be applied to the MTL setting. Indeed, the SELF estimator at \eqref{eq:loss-trick} requires knowledge of all tasks outputs $y_i\in\Y = \R^T$ for every training input $x_i\in\X$ while in MTL we have a separate dataset $(x_{it},y_{it})_{i=1}^{n_t}$ for each task, with $y_{it}\in\R$ (this could be interpreted as the vector $y_i$ to have ``missing entries''). Therefore, in this work we extend the SELF framework to nonlinear MTL. We begin by proving a characterization of the minimizer $f^*:\X\to\C$ of the expected risk $\E(f)$ akin to \eqref{eq:struct_solution}.
\begin{restatable}{proposition}{PNLMTLSolution}\label{prop:nl-mtl-solution}
Let $\ell:\R\times\R\to\R$ be SELF, with $\ell(y,y') = \ip{\psi(y)}{V\psi(y')}_\H$. Then, the expected risk $\E(f)$ introduced at \eqref{eq:expected_risk} admits a measurable minimizer $f^*:\X\to\C$. Moreover, any such minimizer satisfies almost everywhere on $\X$
\eqal{\label{eq:nl-mtl-solution}
    f^*(x) = \argmin_{c\in\C} \sum_{t=1}^T ~\ip{\psi(c_t)}{V g_t^*(x)}_\H, \qquad \textrm{with} \qquad g_t^*(x) = \int_\R \psi(y) ~d\rho_t(y|x).
}
\end{restatable}
\autoref{prop:nl-mtl-solution} extends \eqref{eq:struct_solution} by relying on the linearity induced by the SELF assumption combined with {\em Aumann's principle} \cite{steinwart2008}, which guarantees the existence of a measurable selector $f^*$ for the minimization problem at \eqref{eq:nl-mtl-solution} (see Appendix). By following the strategy outlined in \secref{sec:self-overview}, we propose to learn $T$ independent functions $\widehat g_t:\X\to\H$, each aiming to approximate the corresponding $g_t^*:\X\to\H$ and then define $\widehat f:\X\to\C$ such that for every $x\in\X$
\eqal{\label{eq:nl-mtl-estimator-primal}
  \widehat f(x) = \argmin_{c\in\C} \sum_{t=1}^T \ip{~\psi(c_t)~}{~V ~\widehat g_t(x)~}_\H.
}
Here we consider the $\widehat g_t$ to be learned by solving the $T$ independent kernel ridge regression problems
\eqal{\label{eq:krr}
  \minimize{g\in\H\otimes\mathcal{G}}  \frac{1}{n_t} \sum_{i=1}^{n_t} ~ \|g(x_{it}) - \psi(y_{it})\|^2 + \lambda_t \|g\|_{\H \otimes \mathcal{G}}^2
}
for $t=1,\dots,T$, where $\mathcal{G}$ is a reproducing kernel Hilbert space on $\X$ associated to a kernel $k:\X\times\X\to\R$ and the candidate solution $g:\X\to\H$ is an element of $\H\otimes\mathcal{G}$. The following result shows that in this setting the evaluation of the estimator $\widehat f$ can be significantly simplified.
\begin{restatable}[The Nonlinear MTL Estimator]{proposition}{PLossTrick}\label{prop:nl-mtl-loss-trick}
Let $k:\X\times\X\to\R$ be a reproducing kernel with associated reproducing kernel Hilbert space $\mathcal{G}$. Let $\widehat g_t:\X\to\H$ be the solution of \eqref{eq:krr} for $t=1,\dots,T$. Then the estimator $\widehat f:\X\to\C$ defined at \eqref{eq:nl-mtl-estimator-primal} is equivalently characterized by
\eqal{\label{eq:estimator-multitask}
  \widehat f(x) = \argmin_{c\in\C}~ \sum_{t=1}^T\sum_{i=1}^{n_t} \alpha_{it}(x) \ell(c_t,y_{it}), \quad (\alpha_{1t}(x),\dots,\alpha_{n_t t}(x))^\top = (K_t + n_t \lambda_t I)^{-1} K_{tx}
}
for all $x\in\X$ and $t=1,\dots,T$, where $K_t\in\R^{n_t \times n_t}$ denotes the kernel matrix of the $t$-th task, namely $(K_t)_{ij} = k(x_{it},x_{jt})$, and $K_{tx}\in\R^{n_t}$ the vector with $i$-th component equal to $k(x,x_{it})$.
\end{restatable}
\autoref{prop:nl-mtl-loss-trick} provides the estimator for nonlinear MTL and extends the SELF approach (indeed for VVR, \eqref{eq:estimator-multitask} reduces to the SELF estimator at \eqref{eq:loss-trick}). Interestingly, the proposed strategy learns the score functions $\alpha_{im}:\X\to\R$ separately for each task and then combines them in the joint minimization over $\C$. This can be interpreted as the estimator weighting predictions according to how ``reliable'' each task is on the input $x\in\X$. We make this intuition more clear in the following.
\begin{example}[Nonlinear MTL with Square Loss]\label{ex:ls-and-mtl}
Let $\ell$ be the square loss. Then, analogously to \autoref{ex:ls-and-C} we have that for any $x\in\X$, the multitask estimator at \eqref{eq:estimator-multitask} is
\eqal{
    \widehat f(x) = \argmin_{c\in\C}~ \sum_{t=1}^T a_t(x) \big(c_t - b_t(x)/a_t(x) \big)^2
}
with $a_t(x) = \sum_{i=1}^{n_t} \alpha_{it}(x)$ and $b_t(x) = \sum_{i=1}^{n_t} \alpha_{it}(x) y_{it}$, which corresponds to perform the projection $\widehat f(x) = \Pi_\C^{_{A(x)} }(w(x))$ of the vector  $w(x) = (b_1(x)/a_1(x),\dots,b_T(x)/a_T(x))$ according to the metric deformation induced by the matrix $A(x) = \textrm{diag}(a_1(x),\dots,a_T(x))$. This suggests to interpret $a_t(x)$ as a measure of confidence of task $t$ with respect to $x\in\X$. Indeed, tasks with small $a_t(x)$ will affect less the weighted projection $\Pi_\C^{_{A(x)} }$.
\end{example}
%
%
%
%
\subsection{Extensions: Violating $\C$}\label{sec:violating}

%
%
In practice, it is natural to expect the knowledge of the constraints set $\C$ to be not exact, for instance because of some inaccuracy when modeling it or because of noise.  
To address this issue, we consider 
two extensions of nonlinear MTL
that allow candidate predictors to slightly violate the constraints $\C$ and introduce a hyperparameter to control this effect.\\

\noindent{\bf Robustness~w.r.t.~perturbations of $\C$.} We soften the effect of the constraint set by requiring candidate predictors to take value {\em within} a radius $\delta>0$ from $\C$, namely $f:\X\to\C_\delta$ with
\begin{equation}\label{eq:ball}
\C_\delta = \{~ c + r ~|~ c\in\C, r\in\R^T, \|r\|\leq\delta ~ \}.
\end{equation} 
The scalar $\delta>0$ is now a hyperparameter ranging from $0$ ($\C_0 = \C$) to $+\infty$ ($\C_\infty = \R^T$).\\

\noindent{\bf Penalizing w.r.t. the distance from $\C$.} We can penalize predictions depending on their distance from the set $\C$ by introducing a perturbed version $\ell_\mu^t:\R^T\times\R^T\to\R$ of the original loss, that is
\begin{equation}\label{eq:loss}
\ell_\mu^t(y,z) = \ell(y_t,z_t) + \|z - \Pi_C(z)\|^2/\mu \qquad \textrm{for all} ~ y,z\in\R^T
\end{equation}
where $\Pi_\C:\R^T\to\C$ denotes the orthogonal projection onto $\C$ (see \autoref{ex:ls-and-C}). Below we report the closed-from solution for nonlinear vector-valued regression with square loss. For the sake of brevity, we address the MTL case in the Appendix.
\begin{example}[Vector-valued Regression and Violations of $\C$]\label{ex:ls-and-perturbations}
With the same notation as \autoref{ex:ls-and-C}, let $f_0:\X\to\C$ denote the solution at \eqref{eq:ls-and-C-solution} of nonlinear VVR with {\em exact} constraints, let $r = b(x)/a(x) - f_0(x)\in\R^T$. Then, the solutions to the problem with robust constraints $\C_\delta$ and perturbed loss function $\L_\mu = \frac{1}{T} \sum_t \ell_\mu^t$ are respectively (see Appendix)
\eqal{
    \widehat f_\delta(x) = f_0(x) + r ~ \min(1,\delta/\|r\|) \qquad \textrm{and} \qquad \widehat f_\mu(x) = f_0(x) + r ~ \mu/(1 + \mu).
}
\end{example}
%
\section{Generalization Properties of MTL with Nonlinear Constraints}\label{sec:theory}
We study here the generalization properties of the nonlinear MTL estimator. Our analysis revolves around the assumption that the loss function used to measure prediction errors is SELF. To this end we observe that most loss functions typically used in multitask learning are indeed SELF.
\begin{restatable}{proposition}{PSmooth}\label{teo:smooth}
Let $\bar\ell:[a,b]\to\R$ be differentiable almost everywhere with derivative Lipschitz continuous almost everywhere. Let $\ell:[a,b]\times[a,b]\to\R$ be such that $\ell(y,y') = \bar\ell(y-y')$ or $\ell(y,y') = \bar\ell(yy')$ for all $y,y'\in\R$. Then: $(i)$ $\ell$ is SELF and $(ii)$ the separable function $\L:\Y^T\times\Y^T\to\R$ such that $\L( y, y') = \frac{1}{T}\sum_{t=1}^{ _T} \ell(y_t, y_t')$ for all $y,y'\in\Y^T$ is SELF.
\end{restatable}
Interestingly, most (mono-variate) loss functions used in multitask and supervised learning satisfy the assumptions of \autoref{teo:smooth}. Typical examples are the square loss $(y-y')^2$, hinge $\max(0,1 - yy')$ or logistic $\log(1 - \exp(-yy') )$. Indeed, the derivative of these functions with respect to $z = y-y'$ or $z = yy'$ exists almost everywhere and it is Lipschitz almost everywhere on compact sets.
%
%

The nonlinear MTL estimator introduced in \secref{sec:nonlinear-mtl} relies on the intuition that if for each $x\in\X$ the kernel ridge regression solutions $\widehat g_t(x)$ are close to the conditional expectations $g_t^*(x)$, then also $\widehat f(x)$ will be close to $f^*(x)$. The following result formally characterizes the relation between the two problems, proving what is often referred to as a {\em comparison inequality} in the context of surrogate frameworks \cite{bartlett2006}. Throughout the rest of this section we assume $\rho_t(x,y) = \rho_t(y|x)\rho_\X(x)$ for each $t=1,\dots,T$ and denote $\|g\|_{L^2(\X,\H,\rho_\X)}$ the $L^2$ norm of a function $g:\X\to\H$ according to $\rho_\X$.
\begin{restatable}[Comparison Inequality]{theorem}{TComparison}\label{teo:comparison-inequality}
With the same assumptions of \autoref{prop:nl-mtl-solution}, for $t=1,\dots,T$ let $f^*:\X\to\C$ and $g_t^*:\X\to\H$ be defined as in \eqref{eq:nl-mtl-solution}, let $\widehat g_t:\X\to\H$ be measurable functions and let $\widehat f:\X\to\C$ satisfy \eqref{eq:nl-mtl-estimator-primal}. Let $L^2 = L^2(\X,\H,\rho_\X)$ and $V^*$ be the adjoint of $V$. Then,
\eqal{\label{eq:comparison-inequality}
  \E(\widehat f) - \E(f^*) \leq  q_{\C,\ell,T} \sqrt{\frac{1}{T} \sum_{t=1}^T \|\widehat g_t - g_t^* \|^2_{L^2} }, \quad q_{\C,\ell,T}  = 2 \sup_{c\in\C} \sqrt{\frac{1}{T} \sum_{t=1}^T \|V^*\psi(c_t)\|_\H^2}.
}
\end{restatable}
The comparison inequality at \eqref{eq:comparison-inequality} is key to study the generalization properties of the proposed estimator for nonlinear MTL and indeed shows that we can control the corresponding {\em excess risk} in terms of how well the kernel ridge regression solutions $\widehat g_t$ approximate the true $g_t^*$ (see Appendix for a proof of \autoref{teo:comparison-inequality}). 
\begin{restatable}{theorem}{Tuniversality}\label{teo:universality}
Let $\C\subseteq [a,b]^T$, let $\X$ be a compact set and $k:\X\times\X\to\R$ a continuous universal reproducing kernel\footnote{Standard assumption for universal consistency (see \cite{steinwart2008}). E.g. the Gaussian $k(x,x') = \exp({-\frac{\|x-x'\|^2}{\sigma^2}})$.}. Let $\ell:[a,b]\times[a,b]\to\R$ be a SELF. Let $\widehat f_N:\X\to\C$ denote the estimator at \eqref{eq:estimator-multitask} with $N = (n_1,\dots,n_T)$ training points independently sampled from $\rho_t$ for each task $t=1,\dots,T$ and $\lambda_t = n_t^{-1/4}$. Let $n_0 = \min_{1\leq t \leq  T} n_t$. Then, with probability $1$,
\begin{align}\label{eq:universality}
  \lim_{n_0\to+\infty} \E(\widehat f_N) = \inf_{f:\X\to\C}\E(f).
\end{align}
\end{restatable}
The proof of \autoref{teo:universality} relies on the comparison inequality in \autoref{teo:comparison-inequality}, which links the excess risk of the MTL estimator to the square error between $\hat g_t$ and $g^*_t$. Standard results from kernel ridge regression allow to conclude the proof~\cite{caponnetto2007} (see a more detailed discussion in the Appendix). By imposing further standard assumptions, we can also obtain generalization bounds on $\|\widehat g_t - g_t^*\|_{L^2}$ that automatically apply to nonlinear MTL again via the comparison inequality, as shown below.
\begin{restatable}{theorem}{Tbound}\label{teo:bound}
With the same assumptions of \autoref{prop:nl-mtl-solution} let $\widehat f_N:\X\to\C$ denote the estimator at  \eqref{eq:estimator-multitask} with $\lambda_t = n_t^{-1/2}$ and assume $g_t^*\in\H\otimes\mathcal G$, for all $t=1,\dots,T$. Then for any $\tau > 0$ we have, with probability at least $1 - 8 e^{-\tau}$, that
\eqal{\label{eq:bound}
    \E(\widehat{f}_{N}) - \inf_{f:\X\to\C} \E(f) ~~ \leq ~~ q_{\C,\ell,T} ~~ h_\ell ~~ \tau^2 ~~ n_0^{-1/4} \log T,
}
where $q_{\C,\ell,T}$ is defined as in \eqref{eq:comparison-inequality} and $h_\ell$ is a constant independent of~$\C, N, n_t, \la_t, \tau, T$.
\end{restatable}
%
The the excess risk bound in \autoref{teo:bound} is comparable to that in \cite{ciliberto2016} (Thm.~$5$). To our knowledge this is the first result studying the generalization properties of a learning approach to MTL with 
constraints.\\

\noindent{\bf Benefits of Nonlinear MTL}. The rates in \autoref{teo:bound} strongly depend on the constraints $\C$ via the constant $q_{\C,\ell,T}$. The following result studies two special cases that allow to appreciate this effect.
\begin{restatable}{lemma}{LQForSphere}\label{lem:q-for-sphere}
Let $B\geq1$, $\mathcal B = [-B,B]^T$, $\mathcal S \subset\R^T$ be the sphere of radius $B$ centered at the origin and let $\ell$ be the square loss. Then $q_{\mathcal B,\ell,T} \leq 2\sqrt{5} ~ B^2$ and $q_{\mathcal S,\ell,T} \leq 2\sqrt{5} ~ B^2 ~T^{-1/2}$.
\end{restatable}

To explain the effect of $\C$ on MTL, define $n=\sum_{t=1}^{_T} n_t$ and assume that $n_0 = n_t = n/T$.
\autoref{lem:q-for-sphere} together with Thm.~\ref{teo:bound} shows that when the tasks are assumed not to be related (i.e.~$\C = \B$) the learning rate of nonlinear MTL is of $\widetilde{O}((\frac{T}{n})^{1/4})$, as if the tasks were learned independently. On the other hand, when the tasks have a relation (e.g. $\C = \mathcal S$, implying a quadratic relation between the tasks) 
nonlinear MTL achieves a learning rate of $\widetilde{O}((\frac{1}{n T})^{1/4})$, which improves as the number of tasks increases and as the total number of observed examples increases. Specifically, for $T$ of the same order of $n$, we obtain a rate of $\widetilde{O}(n^{-1/2})$ which is comparable to the optimal rates available for kernel ridge regression {\em with only one task trained on the total number $n$ of examples}~\cite{caponnetto2007}. This observation corresponds to the intuition that if we have many related tasks with few training examples each, we can expect to achieve significantly better generalization by taking advantage of such relations rather than learning each task independently.
\section{Connection to Previous Work: Linear MTL}\label{sec:previous-work}
In this work we formulated the nonlinear MTL problem as that of learning a function $f:\X\to\C$ taking values in a set 
of constraints $\C\subseteq\R^T$ {\em implicitly} identified by a set of equations $\gamma(f(x)) = 0$. An alternative approach would be to characterize the set $\C$ via an {\em explicit} parametrization $\theta:\R^Q\to\C$,~for $Q\in \N$,~so that the multitask predictor can be decomposed as $f = \theta \circ h$, with $h:\X\to\R^Q$. We can learn $h:\X\to\R^Q$ using empirical risk minimization strategies such as Tikhonov regularization, 
\eqal{ \label{eq:alternative-formulation}
    \minimize{ \substack{ h = (h_1,\dots,h_Q) \in \H^Q} } ~~ \frac{1}{n} \sum_{i=1}^{n} \L(\theta \circ h(x_i), y_i) + \lambda \sum_{q=1}^Q \|h_q\|_\H^2
}
since candidate $h$ take value in $\R^Q$ and therefore $\H$ can be a standard linear space of hypotheses. However, while \eqref{eq:alternative-formulation} is interesting from the modeling standpoint, it also poses several problems: $1$) $\theta$ can be nonlinear or even non-continuous, making \eqref{eq:alternative-formulation} hard to solve in practice even for $\L$ convex; $2$) $\theta$ is not uniquely identified by $\C$ and therefore different parametrizations may lead to very different $\widehat f = \theta \circ \widehat h$, which is not always desirable; $3$) There are few results on empirical risk minimization applied to generic loss functions $\L(\theta (\cdot), \cdot)$ (via so-called oracle inequalities, see \cite{steinwart2008} and references therein), and it is unclear what generalization properties to expect in this setting. A relevant exception to the issues above is the case where $\theta$ is {\em linear}. In this setting \eqref{eq:alternative-formulation} becomes more amenable to both computations and statistical analysis and indeed most previous MTL literature has been focused on this setting, either by designing ad-hoc output metrics \cite{sindhwani13}, linear output encodings \cite{fergus10} or regularizers \cite{micchelli2004}. Specifically, in this latter case the problem is cast as that of minimizing the functional 
\eqal{\label{eq:matrix-A}
\minimize{ \substack{ f = (f_1,\dots,f_T) \in \H^T } } ~~ \sum_{i=1}^n \L(f(x_{i}),y_{i}) +\la \sum_{t,s = 1}^T A_{ts} \ip{f_t}{f_s}_\H
}
where the psd matrix $A=(A_{ts})_{s,t=1}^T$ 
encourages linear relations between the tasks. It can be shown that 
this problem is equivalent to \eqref{eq:alternative-formulation} when the $\theta\in\R^{T \times Q}$ is linear and $A$ is set to the pseudoinverse of $\theta\theta^\top$. As shown in \cite{ciliberto2015}, a variety of situations are recovered considering the approach above, such as the case where tasks are centered around a common average ~\cite{evgeniou05}, clustered in groups \cite{jacob08} or sharing the same subset of features \cite{argyriou2007,obozinski2010}. Interestingly, the above framework can be further extended to estimate the structure matrix $A$ directly from data, an idea initially proposed in \cite{dinuzzo11} and further developed in \cite{AlvLawRos12,ciliberto2015,jawanpuria2015}.


\begin{figure}[t]
\centering
\qquad\quad~~~\includegraphics[width=0.2\columnwidth]{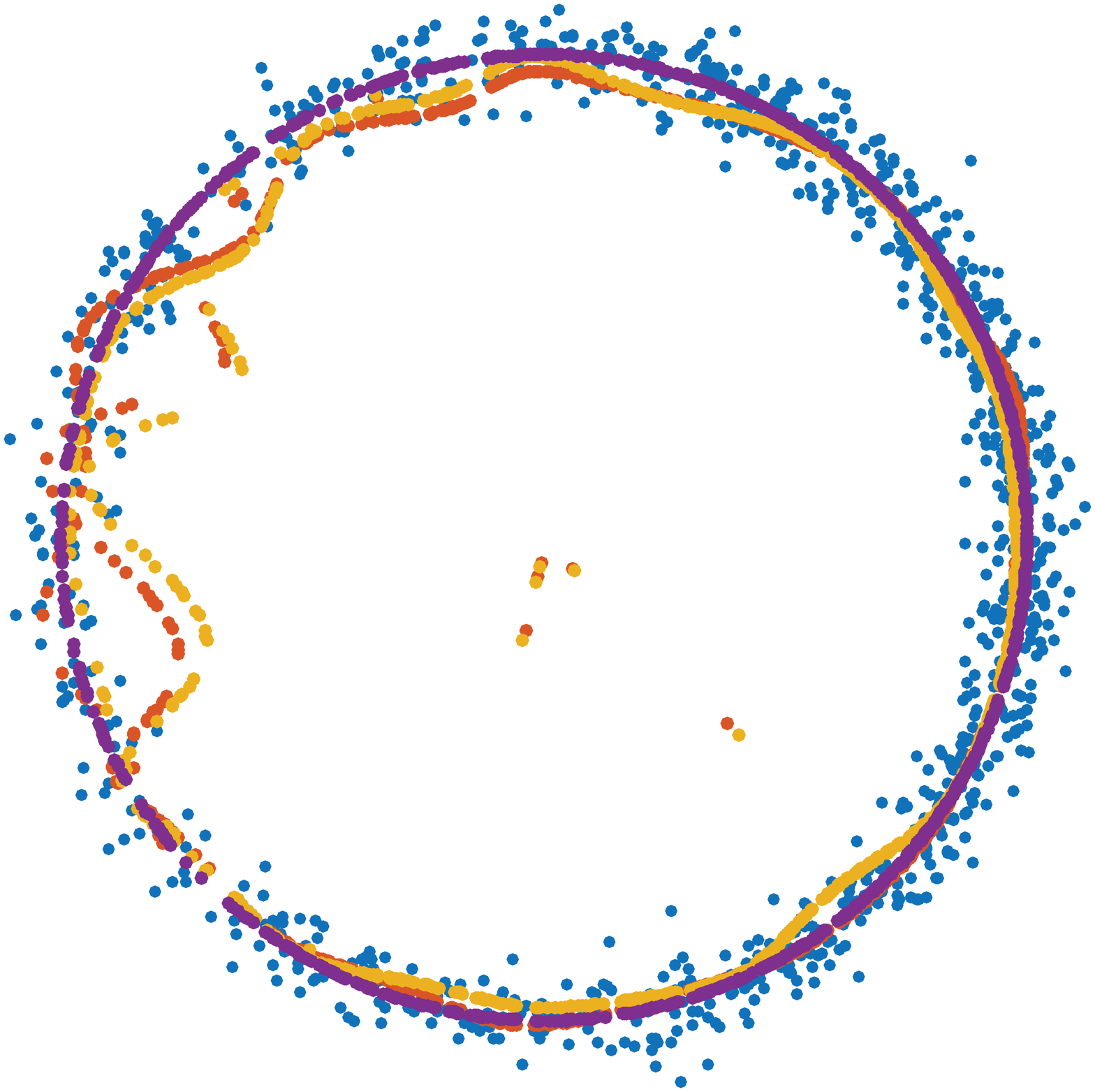}\qquad\qquad\qquad\qquad\qquad\qquad%
\includegraphics[width=0.27\columnwidth]{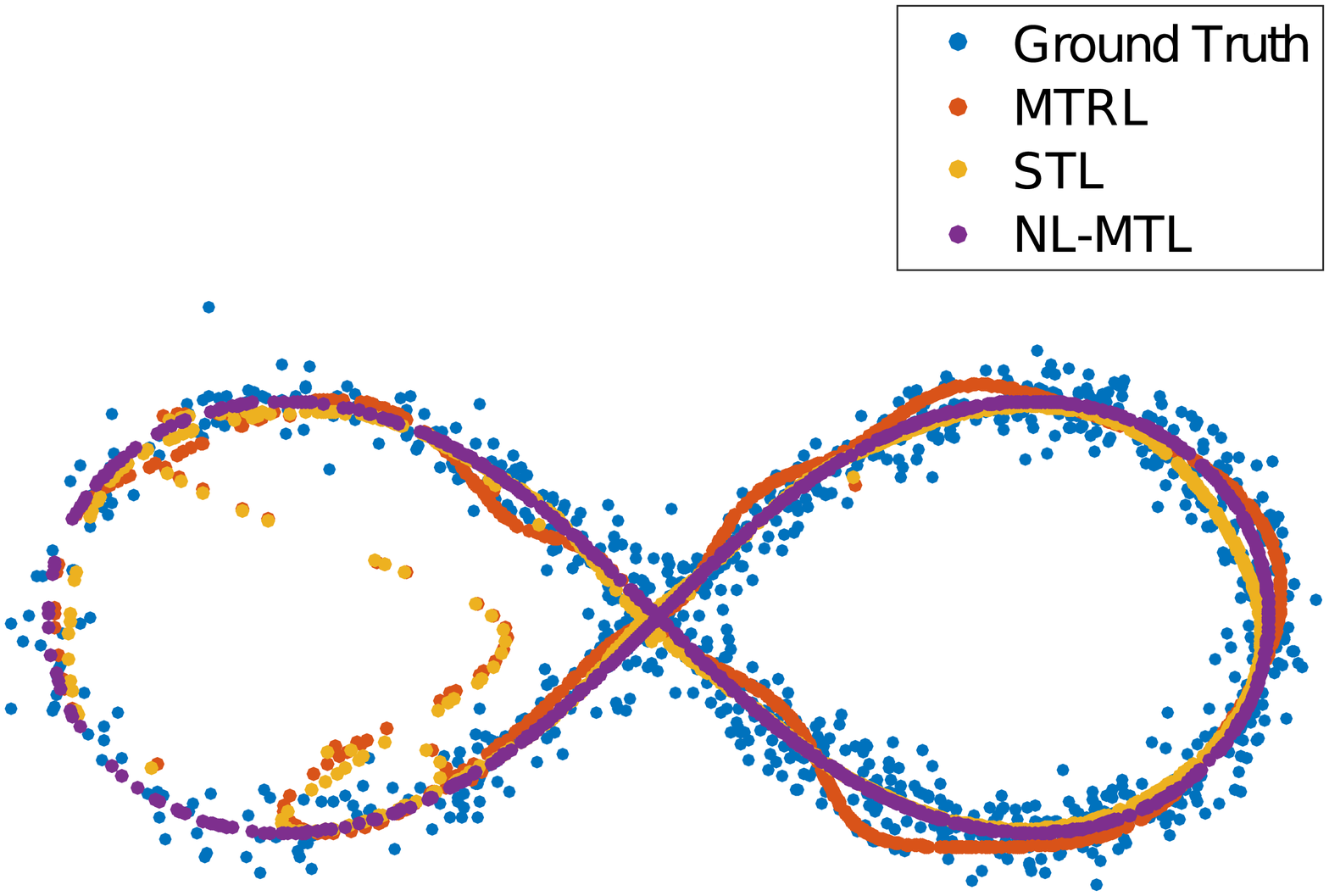}

\includegraphics[width=0.43\columnwidth]{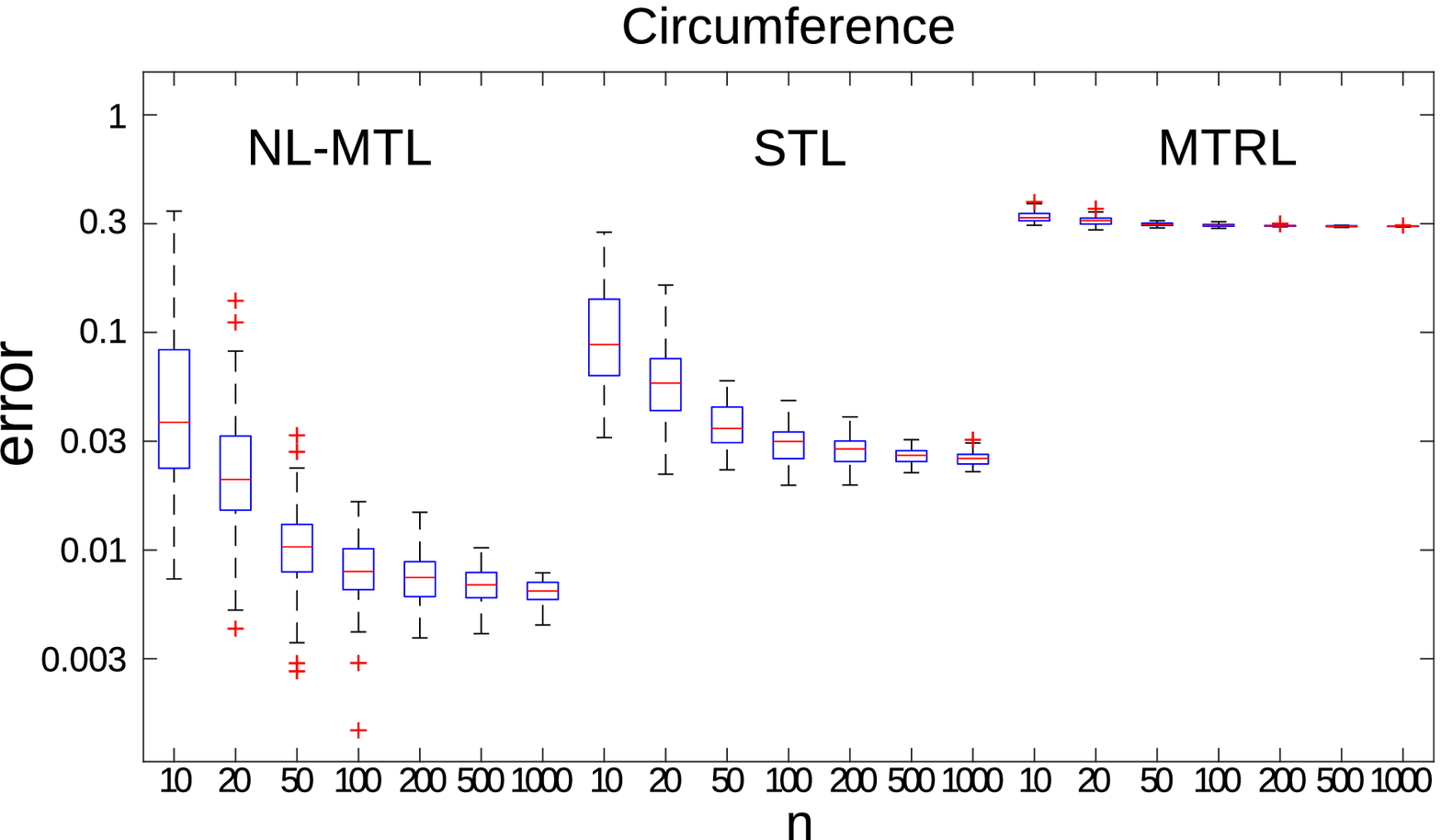}\qquad%
\qquad\includegraphics[width=0.43\columnwidth]{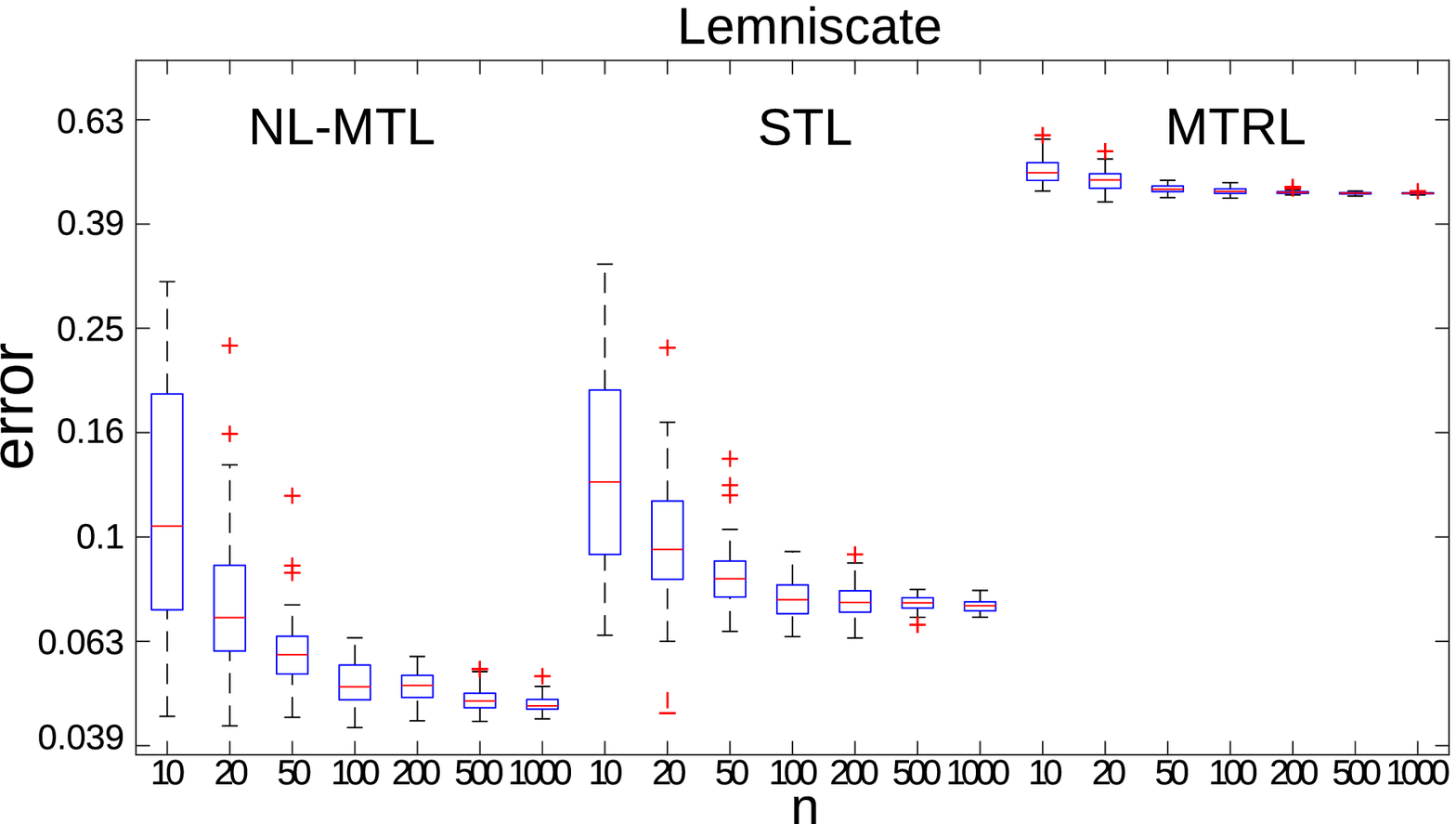}\qquad\qquad
\caption{({\bf Bottom}) MSE (logaritmic scale) of MTL methods for learning constrained on a circumference (Left) or a Lemniscate (Right). Results are reported in a boxplot across $10$ trials. ({\bf Top}) Sample predictions of the three methods trained on $100$ points and compared with the ground truth.}
\label{fig:synth-boxplot}
\end{figure}

\section{Experiments}
\label{sec:experiments}
{\bf Synthetic Dataset}. We considered a model of the form $y = f^*(x) + \epsilon$ with $f^*:\X\to\C$ and $\epsilon \sim \N(0,\sigma I)$ noise sampled according to a normal isotropic distribution. We considered the case of two tasks constrained on $\C\subset\R^2$, either a circumference or a lemniscate (see Fig.~\ref{fig:synth-boxplot}), identified by the equations $\gamma_{\textrm{circ}}(y) = y_1^2+y_2^2 - 1 = 0$ and $\gamma_{\textrm{lemn}}(y) = y_1^4 - (y_1^2 - y_2^2) = 0$
with $y = (y_1,y_2)\in\R^2$. We set $\X = [-\pi,\pi]$ and chose $f^*:[-\pi,\pi]\to\R^2$ to be parametric functions associated to the circumference and Lemniscate respectively, namely $f_{\textrm{circ}}^*(x) = (\cos(x),~ \sin(x))$ and $f_{\textrm{lemn}}^*(x) = (\sin(x), ~ \sin(2x))$.
We sampled from $10$ to $1000$ points for training and $1000$ for testing, with noise $\sigma = 0.05$.

We trained and tested different models over $10$ trials. We used a Gaussian kernel on the input and chose the corresponding bandwidth and the regularization parameter $\lambda$ by hold-out cross-validation on $30\%$ of the training set for $30$ values of the two hyperparameters sampled respectively from the ranges $[0.01,100]$ and $[1e-9,1]$ with logarithmic spacing. Fig.~\ref{fig:synth-boxplot} (Bottom) reports the mean square error (MSE) of the nonlinear MTL approach (NL-MTL) compared with the standard least squares single task learning (STL) baseline and the multitask relations learning (MTRL) from \cite{zhang10}, which encourages tasks to be linearly dependent. For both the circumference and Lemniscate, the tasks are strongly {\em nonlinearly} related. As a consequence our approach consistently outperforms its two competitors which assume only linear relations (or none at all). Fig.~\ref{fig:synth-boxplot} (Top) provides a qualitative comparison on the three methods (when trained with $100$ examples) during a single trial.\\
\begin{table}[t]
\scriptsize
\caption{Explained variance of the robust (NL-MTL[R]) and perturbed (NL-MTL[P]) variants of nonlinear MTL, compared with linear MTL methods on the Sarcos dataset  reported from \cite{jawanpuria2015}.}\label{tab:real-data}
\begin{center}
\begin{tabular}{rcccccccc}
\toprule

                           &  {\bf STL} & {\bf MTL}\cite{evgeniou2004} & {\bf CMTL}\cite{jacob08} & {\bf MTRL}\cite{zhang10} & {\bf MTFL}\cite{jawanpuria2012} & {\bf FMTL}\cite{jawanpuria2015} & {\bf NL-MTL[R]} & {\bf NL-MTL[P]} \\
\midrule
{\bf Expl.}                         & $40.5$ & $34.5$ & $33.0$ & $41.6$ & $49.9$ & $50.3$ & $\mathbf{55.4}$ & $54.6$ \\
{\bf Var. (\%)} & $\pm 7.6$ & $ \pm  10.2$ & $\pm  13.4$ & $\pm  7.1$ & $\pm  6.3$ & $\pm 5.8$ & $\mathbf{\pm 6.5}$ & $\pm 5.1$ \\
\hline
\end{tabular}
\end{center}
\end{table}
%

\noindent{\bf Sarcos Dataset}. We next report experiments on the Sarcos dataset~\cite{rasmussen2006}. The goal is to predict the torque measured at each joint of a $7$ degrees-of-freedom robotic arm, given the current state, velocities and accelerations measured at each joint ($7$ tasks/torques for $21$-dimensional input). We used the $10$ dataset splits available online for the dataset in \cite{jawanpuria2012}, each containing $2000$ examples per task with $15$ examples used for training/validation while the rest is used to measure errors in terms of the {\em explained variance}, namely $1$ - nMSE (as a percentage). In order to be fair with results in \cite{jawanpuria2012} we used the canonical linear kernel on the input. We chose the hyperparamters each from $30$ values sampled in the ranges $[1e^{-5},1e^5]$ for $\delta$ and $\mu$ and $[1e^{-9},1]$ for $\lambda$ with logarithmic spacing.

Table~\ref{tab:real-data} reports results from \cite{jawanpuria2012,jawanpuria2015} for a wide range of previous {\em linear} MTL methods \cite{evgeniou2004,jacob08,argyriou2007,zhang10,jawanpuria2012,jawanpuria2015}, together with our nonlinear MTL approach (both the robust and perturbed versions). Since, we were not able to find the model parameters of the Sarcos robot online, we approximated the constraint set $\C$ as a point cloud by collecting $1000$ random output vectors that neither belonged to the training nor to the test sets by sampling them from the original sarcos dataset \cite{rasmussen2006}. As can be noticed, the approach proposed in this paper clearly outperforms the ``linear'' competitors. Note that indeed, the relations among the torques measured at different joints of a robot are highly nonlinear (see for instance \cite{sciavicco1996}) and therefore taking such structure into account can be beneficial to the learning process.\\

\noindent{\bf Ranking by Pair-wise Comparison.} We considered a ranking prediction problem, formulated within a non-linear multi-task learning setting. In particular, given a database of $D$ documents, the goal is to learn $T = D(D-1)/2$ functions $f_{p,q}:\X\to\{-1,0,1\}$, one for each pair of documents $p$ and $q$ from $1$ to $D$, predicting if one document is more relevant than the other for a given input query $x$. Recommender systems are a natural application in this setting, and here we discuss an experiment on the Movielens100k \cite{harper2015} dataset for movie recommendation, where movies correspond to documents and queries to users. 

The problem can be tackled as multi-label MTL with $0$-$1$ loss: for a given training query $x$ only some pairwise comparisons $y_{p,q}\in\{-1,0,1\}$ are available (with $1$ corresponding to movie $p$ being more relevant than $q$, $-1$ the opposite and $0$ that the two movies are equivalent). Errors are measured according to the loss function
\eqal{\label{eq:binary_ranking}
  \loss_{pairwise}(c,y) = \sum_{p = 1}^{D-1}\sum_{q=p+1}^D ~\boldsymbol 1_{\{c_{p,q} ~\neq~ y_{p,q}\} }
}
with $\boldsymbol 1_{\{c_{p,q} \neq y_{p,q}\}} = 1$ if $c_{p,q} \neq y_{p,q}$ and zero otherwise. We have  therefore $T$ separate training sets, one for each task (i.e. pair of movies). Clearly, not all possible combinations of outputs $f:\X\to\{-1,0,1\}^T$ are allowed: the comparisons need to be logically consistent (e.g. if $p \succ q$ (read ``$p$ more relevant than $q$'') and $q \succ r$, then we cannot have $r \succ p$). As shown in \cite{duchi2010} these constraints are naturally encoded in a set $\mathcal C = DAG(D)$ in $\R^T$ of all vectors $G\in\R^T$ that correspond to (the vectorized, upper triangular part of the adjacency matrix of) a Directed Acyclic Graph with $D$ vertices. This leads to the NL-MTL estimator, 
\eqal{\label{eq:estimator_binary_ranking}
\hat{f}(x) = \argmin_{c \in DAG(D)} ~ \sum_{i=1}^n \sum_{p=1}^{D-1}\sum_{q = p + 1}^D \alpha_{i,(p,q)}(x) ~\boldsymbol 1_{\{c_{p,q} ~\neq~y_{p,q}\}}
}
where each score function $\alpha_{i,(p,q)}(x)$ is learned according to \eqref{eq:estimator-multitask} using only the training inputs for which the pairwise comparison $y_{p,q}$ is available. As already observed in \cite{ciliberto2016}, the above optimization over $DAG(D)$ can be formulated as a {\em minimal feedback Arc Set} problem on graphs \cite{slater1961} that can be addressed using approximation methods \cite{eades1993}.

As a final remark, note that in the ranking prediction settings of \cite{duchi2010,ciliberto2016}, the authors considered a {\em weighted} version of the binary loss at \eqref{eq:binary_ranking} of the form
$$
\loss_{rank}(c,y) = \sum_{p = 1}^{D-1}\sum_{q = p + 1}^D  |y_{p,q}| ~\boldsymbol 1_{\{c_{p,q}~ \neq ~\textrm{sign}(y_{p,q})\}}
$$
where $y_{p,q}\in\R$ is a scalar value measuring the discrepancy between documents $q$ and $q$ and is used to weight misclassification errors according to the relative importance between documents $p$ and $q$. Indeed, in the Movielens100k dataset, each movie $p$ was assigned a rating by the user (a discrete value $r_p$ from $1$ to $5$) and in \cite{duchi2010} the weights in $\loss_{rank}$ were chosen as $y_{p,q} = r_p - r_q$. It is clear that this variant of $\loss_{pairwise}$ leads to an NL-MTL estimator analogous to the one in \eqref{eq:estimator_binary_ranking}, for which the optimization over $\C = DAG(D)$ can still be tackled with minimal feedback arc set solvers. In particular for our experiments we used the publicly avaialable igraph library\footnote{\url{http://igraph.org/}} which implements a number of standard algorithms to solve problems on graphs. 

We performed experiments on Movielens100k to compare our MTL estimator with both standard multi-task learning baseline as well as with methods specifically designed to address ranking problem. In particular, we used the linear kernel and the train, validation and test splits adopted in \cite{ciliberto2016} to perform $10$ independent trials with 5-fold cross-validation for model selection. In Tab.~\ref{tab:ranking} we report the average ranking error and standard deviation of the weighed $0$-$1$ loss function considered in \cite{duchi2010,ciliberto2016} for the ranking methods proposed in \cite{herbrich1999,dekel2004,duchi2010}, the SVMStruct estimator \cite{tsochantaridis2005}, the SELF estimator considered in \cite{ciliberto2016} for ranking, the MTRL predictor and the STL baseline, corresponding to individual SVMs trained for each pairwise comparison. Results for previous methods are reported from \cite{ciliberto2016}.

\begin{table}[t]
\scriptsize
\caption{Rank prediction error according to the weighted binary loss in \cite{duchi2010,ciliberto2016}.}\label{tab:ranking}
\begin{center}
\begin{tabular}{rcccccccc}
\toprule

           &  {\bf NL-MTL} & {\bf SELF}\cite{ciliberto2016} & {\bf Linear} \cite{duchi2010} & {\bf Hinge} \cite{herbrich1999} & {\bf Logistic} \cite{dekel2004} & {\bf SVMStruct} \cite{tsochantaridis2005} & {\bf STL} & {\bf MTRL}\cite{zhang10}\\
\midrule
{\bf Rank} & $\mathbf{0.321}$ & $0.396$ & $0.430$ & $0.432$ & $0.432$ & $0.451$ & $0.581$ & $0.613$ \\
{\bf Loss} & $\pm \mathbf{0.004}$ & $ \pm  0.003$ & $\pm 0.004$ & $\pm 0.008$ & $\pm 0.012$ & $\pm 0.008$ & $0.003$ & $\pm 0.005$ \\
\hline
\end{tabular}
\end{center}
\end{table}

As can be observed, NL-MTL outperforms all competitors, achieving better performance than the the original SELF estimator. Indeed it can be shown that \cite{ciliberto2016} tackles pairwise-based ranking in a vector-valued fashion, by filling missing observations with $0$. This has a negative effect, artificially biasing predictions towards $0$. Conversely, NL-MLT exploits the true MTL nature of the problem leading to superior performance. As expected, STL and MTRL are clearly not suited for this setting, since they miss the critical information contained in the constraint set.

\section{Discussion}\label{sec:discussion}
We studied the problem of learning multiple tasks which are nonlinearly related one to the other. We adopted a structured prediction perspective to address the case where relations identify a set of constraints in the space of all possible outputs. In particular, we extended a recent approach to structured prediction to MTL, deriving a novel algorithm for which we proved universal consistency and generalization bounds. A main direction of future research will be to extend the algorithm presented here in order to learn the {\em nonlinear relations}, in line with previous work from the linear MTL literature. To this end, approaches focused on manifold-approximation or set learning could offer a viable direction for future research.

\section*{Acknowledgments}

This work was supported in part by EPSRC grant EP/P009069/1.

{%
\small
\bibliographystyle{natbib}
\putbib[./example_paper]
}
\end{bibunit}

\newpage

\appendix

\begin{bibunit}[unsrt]

\section*{Appendix}

Here we collect the proofs of the main results presented in the paper.

\section{Theoretical Analysis}

In the following we will assume $\X, \C$ to be Polish spaces, namely separable complete metrizable spaces, equipped with the associated Borel sigma-algebra. In particular we will restrict to the case where $\C$ is a subset of $\Y^T$ with $\Y = [a,b]$ and $a,b \in \R$. For any $y\in\Y^T$ we denote with $y_t$ the $t$-th element of $y$. In the rest of the section, we assume that there exists a probability measure $\rho_\X = \rho$ on $\X$ such that $\rho_t(y,x) = \rho_t(y|x)\rho(x)$ for all $t=1,\dots,T$.

\subsection{Characterization of the Nonlinear MTL Estimator}

In order to prove that a solution of the nonlinear MTL problem has the form described in \autoref{prop:nl-mtl-solution}, namely 
\[
  f^*(x) = \argmin_{c\in\C}~ \sum_{t=1}^T ~ \ip{~ \psi(c_t)~}{V ~ g_t^*(x)~ }_\H \qquad g_t^*(x) = \int_\Y \psi(y)~d\rho_t(y|x)
\]
we start by providing an alternative characterization for the expected risk $\E(f)$ in terms of the $g_t^*:\X\to\H$. 

\begin{lemma}\label{lemma:self-characterization-E}
Let $\ell:\Y\times\Y\to\R$ be SELF with $\ell(y,y') = \ip{\psi(y)}{V\psi(y')}_\H$ for all $y,y'\in\Y$. Then, for any measurable function $f:\X \to \C$, the expected risk $\E(f)$ defined at \eqref{eq:expected_risk} is equal to
\eqal{\label{eq:alternative_expected_risk}
  \E(f) =  \frac{1}{T} \int_\X\sum_{t=1}^T  \left\langle \psi(f_t(x)), ~V~g_t^*(x) \right\rangle_\H d\rho(x)
}
Note that by the definition of SELF (\autoref{asm:self}), $\psi$ is continuous on $\Y$ and therefore $g_t$ is measurable and $\|g^*_t(x)\|_\H$ bounded a.e. on $\X$ for all $t=1,\dots,T$.
\end{lemma}
\begin{proof}
\eqal{
    \E(f)  & =  \frac{1}{T} \sum_{t=1}^T \int_{\X \times \R} \ell(f_t(x), y) ~ d\rho_t(y,x) \\
    & =  \frac{1}{T} \sum_{t=1}^T \int_{\X \times \R}\ip{\psi_t(f_t(x))}{V \psi(y)}_\H ~ d\rho_t(y,x) \\
    & =  \frac{1}{T} \sum_{t=1}^T \int_{\X} \int_\R \left\langle\psi(f_t(x)),~ V \psi(y)\right\rangle_\H ~ d\rho_t(y|x) ~ d\rho(x) \\
    & = \frac{1}{T}\int_\X \sum_{t=1}^T  \left\langle \psi(f_t(x)), ~V\int_\R \psi(y) d\rho_t(y|x) \right\rangle_\H d\rho(x) \\
    & = \frac{1}{T} \int_\X \sum_{t=1}^T \left\langle \psi(f_t(x)), ~V ~g_t^*(x) \right\rangle_\H d\rho(x).
}
\end{proof}

The result above implies that if there exists a function $f^*:\X\to\C$ that minimizes the argument $\sum_{t=1}^T \ip{\psi(f(x)}{V g_t^*}$ in the integral at \eqref{eq:alternative_expected_risk} almost everywhere on $\X$, then $f^*$ minimizes also the expected risk $\E(f)$. The following result guarantees the existence of such a function.


\PNLMTLSolution*

The result is a direct corollary of the following.
\begin{lemma}[$f^*$ is a minimizer of $\cal E$]\label{lemma:fstar-minimize-E}
Let $g_1, \dots, g_T: \X \to \H$ be measurable functions with $\|g_t(x)\|_\H$ bounded a.e. and let $\bar{\E}$ be defined as
$$\bar{\cal E}(f) =  \int_\X r(x, f(x)) d\rho(x), \quad \textrm{with} \quad r(x, c) := \frac{1}{T} \sum_{t=1}^T  \left\langle \psi(c_t), ~V~g_t(x) \right\rangle_\HY$$
If $\ell: \Y \times \Y \to \R$ is continuous, then there exists a measurable selector 
$f^\circ(x) \in \argmin_{c \in \C} r(x,c)$
and a measurable function $m(x) = \min_{c \in \C} r(x,c)$ a.e. such that, we have
$$ \bar{\cal E}(f^\circ) = \int_X m(x) d\rho(x) = \inf_{f: \X \to \C} \bar{\cal E}(f),$$
In particular, by selecting $g_t := g^*_t$ of \eqref{eq:nl-mtl-solution}, we have that $\bar{\E}$ is equal to the expected risk $\E(f)$ in \eqref{eq:expected_risk}, $f^\circ$ is equal to $f^*$ of \eqref{eq:nl-mtl-solution} and minimizes $\E$.
\end{lemma}
\begin{proof}
Note that $r(x,c)$ is Charat\'eodory since $\ell$ is continuous and the $g_t$ are measurable. Then, by the compactness of $\C$, we can invoke {\em Aumann's Measurable Selection Principle} (see e.g. Lemma A.$3$.$18$ in \cite{steinwart2008}), to guarantee that $m$ is measurable and there exists a measurable $f^\circ:\X \to \C$ such that $r(x,f^\circ(x)) = m(x)$ for any $x \in \X$. Moreover, by definition of $m$, we have $r(x, f^\circ(x)) = m(x) \leq r(x,f(x))$ a.e. on $\X$ for any measurable function $f:\X \to \C$.
So we have
$$\bar{\cal E}(f^\circ) = \int_\X r(x, f^\circ(x))d\rho(x)  = \int_\X m(x) d\rho(x) \leq \int_\X r(x,f(x)) d\rho(x) = {\cal E}(f),$$
for any measurable function $f:\X \to \C$. Then ${\cal E}(f) = \inf_{f:\X \to \C} {\cal E}(f)$.
\end{proof}

\noindent{\bf Nonlinear Multitask Learning}. Following the intuition provided in \secref{sec:self-overview} for the original SELF algorithm, a natural approach to approximate $f^*:\X\to\C$ is to consider the estimators $\widehat g_t:\X\to\H$ for the individual $g^*_t:\X\to\H$ and then define $\widehat f:\X\to\C$ such that
\eqal{\label{eq:new-form-fhat}
  \widehat f(x) = \argmin_{c \in \C} \sum_{t=1}^T \left\langle ~\psi(c_t)~, ~V ~\widehat g_t(x)~ \right\rangle_\H
}
In particular, the following Lemma provides the characterization of $\widehat f$ for the case where the $\widehat g_t$ are the solution of kernel ridge regression applied independently for each task $t=1,\dots,T$ with training set $(x_{it},\psi(y_{it}))_{i=1}^{n_t}$
\[
  \widehat g_t(x) = \argmin_{g\in \H \otimes {\cal G}} \frac{1}{n_t} \sum_{i=1}^{n_t} \|\psi(y_{it}) - g(x_{it})\|^2 + \lambda \|g\|_{\H \otimes {\cal G}}^2
\]
where $\mathcal{G}$ is a reproducing kernel Hilbert space with associated reproducing kernel $k:\X\times\X\to\R$. Recall that since $\H$ and $\mathcal{G}$ are reproducing kernel Hilbert spaces, also $\H\otimes\mathcal{G}$ is. 

\PLossTrick*

\begin{proof}
First, note that the solution of the problem above has the form of
$$\widehat g_t(x) = \sum_{i=1}^{n_t} \alpha_{it}(x) \psi(y_{it})$$
with $\alpha_{it}:\X\to\R$ defined as in \eqref{eq:estimator-multitask} (This is a standard result when $\H$ is finite dimensional, see \cite{ciliberto2016} Lemma~$17$ and Eq.$(88)$-$(89)$ for the infinite dimensional case. Then, according to the SELF characterization of the loss $\ell$ in terms of $\psi$ and $V$, we have for any $x \in \X$
\eqal{
\frac{1}{T} \sum_{t=1}^T \sum_{i=1}^{n_t} 
\alpha_{it}(x) \ell(c_t,y_{it}) & = \frac{1}{T} \sum_{t=1}^T \sum_{i=1}^{n_t} 
\alpha_{it}(x) \ip{\psi(c_t)}{V y_{it}}_H \\
& = \frac{1}{T} \sum_{t=1}^T \ip{\psi(c_t)}{V \sum_{i=1}^{n_t} 
\alpha_{it}(x) y_{it}}_H \\
& = \frac{1}{T} \sum_{t=1}^T \ip{\psi(c_t)}{V \widehat{g}_t(x)}_H.
}
\end{proof}

In the following section we study the generalization properties of such estimator as the number of examples per task grows.

\subsection{A Comparison Inequality for Multitask Learning}

In this section we prove the comparison inequality of \autoref{teo:comparison-inequality}, which is key to study the generalization performance of nonlinear MTL. 
 
\TComparison*

\begin{proof}

Recall that for any measurable $g_t:\X\to\H$, 
\[
  \|g_t^* - g_t\|_{L^2(\X,\H,\rho)}^2 := \int \|g_t^*(x) - g_t(x)\|^2_\H d\rho(x).
\]
Let $\bar{\cal E}$ be defined as
$$
\bar{\cal E}(f) = \int_\X \frac{1}{T} \sum_{t=1}^T \left\langle \psi(f(x)_t), ~V~\widehat g_t(x) \right\rangle_\HY d\rho(x),
$$
for any measurable function $f: \X \to \C$. We have
\eqal{
    \E(f) - \E(f^*) = \E(f) - \bar{\cal E}(f) + \bar{\cal E}(f) -\E(f^*) = A + B
}
We will bound separately the terms $A$ and $B$.
By the characterization of ${\cal E}$ in Lemma~\ref{lemma:self-characterization-E}, we have
\eqal{
    A = \E(f) - \bar \E(f) & = \frac{1}{T} \int_\X \sum_{t=1}^T \ip{ \psi(f(x)_t)}{V (g_t^*(x) - \widehat g_t(x))}_\H d\rho(x) \\
    & \leq \frac{1}{T} \int_\X \sum_{t=1}^T \|V^*\psi(f(x)_t)\|_\H \|g_t^*(x) -  \widehat g_t(x)\|_\H d\rho(x) \\
    & \leq  \frac{1}{T} \sum_{t=1}^T \sup_{c \in \C} \|V^*\psi(c_t)\|_\H \int_\X \|g_t^*(x) -  \widehat g_t(x)\|_\H d\rho(x) \\
    & = \frac{1}{T} \sum_{t=1}^T \sup_{c \in \C} \|V^*\psi(c_t)\|_\H \sqrt{\int_\X \|g_t^*(x) -  \widehat g_t(x)\|_\H^2 d\rho(x)} \\
    & =  \frac{1}{T} \sqrt{\sum_{t=1}^T \sup_{c \in \C} \|V^*\psi(c_t)\|^2_\H} \sqrt{\sum_{t=1}^T \int_\X \|g_t^*(x) -  \widehat g_t(x)\|_\H^2 d\rho(x)}.
}
Since, by Lemma~\ref{lemma:fstar-minimize-E} $f^*$ is a minimizer of ${\cal E}$ and with the same reasoning $f$ is a minimizer of $\bar{\cal E}$, we have 
\eqal{
    B &= \bar \E(f) - \E(f^*) \\
    & = \int_\X \min_{c\in\C} \frac{1}{T}\sum_{t=1}^T \ip{\psi(c_t)}{V \widehat g_t(x)}_\H~d\rho(x) ~~-~~ \int\min_{c\in\C} \frac{1}{T}\sum_{t=1}^T \ip{\psi(c_t)}{V g^*_t(x)}_\H ~d\rho(x)\\
    & = \frac{1}{T}\int_\X \min_{c\in\C} \sum_{t=1}^T \ip{\psi(c_t)}{V \widehat g_t(x)}_\H~-~ \min_{c\in\C} \sum_{t=1}^T \ip{\psi(c_t)}{V g^*_t(x)}_\H ~ d\rho(x)\\
    & \leq \frac{1}{T}\int_\X \sup_{c\in\C} \left|\sum_{t=1}^T \ip{\psi(c_t)}{V  \widehat g_t(x)}_\HY - \sum_{t=1}^T \ip{\psi(c_t)}{V g^*_t(x)}_\H\right| ~ d\rho(x)\\
    &= \frac{1}{T}\int_\X \sup_{c\in\C} \left|\sum_{t=1}^T \ip{V^*\psi(c_t)}{(\widehat g_t(x) - g_t^*(x))}_\H\right| ~ d\rho(x)\\
    &\leq \frac{1}{T}\int_\X \sup_{c\in\C} \sum_{t=1}^T \|V^*\psi(c_t)\|_\H ~\|\widehat g_t(x) - g_t^*(x)\|_\H ~ d\rho(x)\\
    &\leq \frac{1}{T}\int_\X \sup_{c\in\C} \sqrt{\sum_{t=1}^T \|V^*\psi(c_t)\|^2_\H} ~\sqrt{\sum_{t=1}^T \|\widehat g_t(x) - g_t^*(x)\|^2_\H} ~ d\rho(x)\\
    &\leq \frac{1}{T} \sqrt{ \sup_{c\in\C}\sum_{t=1}^T \|V^*\psi(c_t)\|^2_\H} \sqrt{\int_\X \sum_{t=1}^T \|\widehat g_t(x) - g_t^*(x)\|^2_\H d\rho(x)}.
}
as desired. 
\end{proof}

\subsection{Consistency and Generalization Bounds}

The comparison inequality provided by \autoref{teo:comparison-inequality} allows to characterize the generalization properties of the estimator $\widehat f:\X\to\C$ by studying the functions $\widehat g_t:\X\to\H$. In particular, if $\widehat g_t \to g_t^*$ in $L_2$ for all $t=1,\dots,T$, the comparison inequality automatically guarantees that the excess risk $\E(\widehat f) - \E(f^*)\to0$. Moreover, if we are able to determine the rate for which the $\widehat g_t$ converge to the $g_t^*$, we can provide generalization bounds for the SELF estimator $\widehat f$.

The following results formalize these ideas, proving consistency and learning rates for $\widehat f$.

\Tuniversality*

\begin{proof}
By the comparison inequality in \autoref{teo:comparison-inequality} it is sufficient to prove that $\widehat g_t\to g_t^*$ in $L^2$ in probability as the number of training point increases. However note that $g_t^*:\X\to\H$ such that $g_t^*(x) = \int_\R \psi(y) ~d\rho_t(y|x)$ is the minimizer of the least squares expected risk 
\begin{equation}\label{eq:least-squares-expected-risk}
  \minimize{g:\X\to\H}~~\int_{\X\times\Y} \|g(x) - \psi(y)\|_\H^2 ~d\rho_t(x,y). 
\end{equation}
Indeed,
\begin{align*}
  \int_{\X\times\Y} \|g(x) - \psi(y)\|_\H^2 ~d\rho_t(x,y) & = \int_{\X\times\Y} \|g(x)\|_\H^2 - 2\ip{g(x)}{\psi(y)} ~d\rho_t(x,y) + const \\
  & = \int_{\X} \|g(x)\|_\H^2 - 2\ip{g(x)}{\int_\Y \psi(y) ~d\rho_t(y|x)} ~ d\rho_\X(x) + const\\
  & = \int_{\X} \|g(x)\|_\H^2 - 2\ip{g(x)}{g_t^*(x)} ~ d\rho_\X(x) + const\\
\end{align*}
Which is clearly minimized pointwise for each $x\in\X$ by $g(x) = g_t^*(x)$.

Therefore, in order to study the convergence of the kernel ridge regression estimator $\hat g_t$ to $g_t^*$ we can leverage on either standard results for kernel ridge regression (e.g. see \cite{caponnetto2007}) if $\H$ is finite dimensional or the result in \cite{ciliberto2016} (Lemma $18$ see Eq.~$96$) when $\H$ is infinite dimensional. Both analyses provide analogous bounds on $\|\widehat g_t - g_t^*\|_{L^2}$ with respect to the number of training examples $n_t$. In particular, the direct application of \cite{ciliberto2016} (Lemma $18$) to our setting, leads to the desired convergence of $\widehat g_t$ to $g_t^*$, and consequently of $\E(\widehat f)$ to $\E(f^*)$ as desired.
\end{proof}

\Tbound*

\begin{proof}
The same reasoning in the proof of \autoref{teo:universality} applies. In particular, since $g_t^*\in\H\otimes\mathcal G$ for every $t=1,\dots,T$, we can improve the bound in \cite[Lemma $18$]{ciliberto2016} analogously to what is done in \cite[Thm. $5$]{ciliberto2016}, leading, for any $\eta>0$, to the bound
\[
  \|g_t - g_t^*\|_{L^2} \leq h_{\ell,t} ~\eta^2 ~n_t^{-1/4}
\]
holding with probability at least $1-8 e^{-\eta}$, where
\[
  h_{\ell,t} = 12(\Psi + \kappa\|g_t^*\|_{\H\otimes\mathcal{G}} + \kappa) \qquad\textrm{with}\qquad \Psi = \sup_{y\in\Y = [a,b]} \|\psi(y)\|_\H ~~\textrm{and}~~ \kappa = \sqrt{\sup_{x\in\X} k(x,x)}
\]
is a constant not depending on $\C,N,n_t,\eta,\lambda_t$. By taking the intersection bound of these events, we have that the following holds with probability $1 - T 8 e^{-\eta}$
\[
  \sqrt{ \frac{1}{T} \sum_{t=1}^T \|g_t - g_t^*\|_{L^2}^2} \leq \sqrt{ \frac{1}{T} \sum_{t=1}^T h_{\ell,t}^2~ \eta^4 ~n_t^{-1/2} } \leq h_\ell~ \eta^2~ n_0^{-1/4}.
\]
with $h_{\ell} = \max_{t=1,\dots,T}~ h_{\ell,t}$. Then, by choosing $\eta := \tau\log(T)$ we have 
\[
  \sqrt{ \frac{1}{T} \sum_{t=1}^T \|g_t - g_t^*\|_{L^2}^2} \leq h_\ell ~ \tau^2 ~n_0^{-1/4}\log(T),
\]
with probability $1 - 8 e^{-\tau}$. Combining the equation above with the comparison inequality we have the desired generalization bound.
\end{proof}

We conclude proving the result studying the constant $q_{\C,\ell,T}$ in the case of $\C = [-B,B]^2$ and $\C = \mathcal S_T$ the sphere of radius B. 

\LQForSphere*

\begin{proof}
We begin by observing that the least square expected risk can be equivalently written as
\eqal{
\E(f) & = \min_{f:\X\to\C}~ \frac{1}{T} \sum_{t=1}^T \int_{\X\times\Y} (f(x) - y)^2 ~~d\rho_t(x,y)\\
 & = \min_{f:\X\to\C}~ \frac{1}{T} \sum_{t=1}^T \int_{\X\times\Y} f(x)^2 - f(x) y + y^2 ~~d\rho_t(x,y) \\
& = \min_{f:\X\to\C}~ \frac{1}{T} \sum_{t=1}^T \int_{\X\times\Y} f(x)^2 - 2 f(x) y ~~d\rho_t(x,y).
}
Therefore we can equivalently study the asymmetric loss $\ell:\R\times\R\to\R$ such that $\ell(y,y') = y^2 - 2 yy'$. Such loss can be written in SELF form $\ell(y,y') = \ip{\psi(y)}{V \psi(y')}_\H$ with $\H = \R^3$, 
\[
  \psi(y) = (y^2, y, 1)^\top \in \R^2 \qquad \textrm{and} \qquad V = \left(\begin{array}{ccc} 0 & 0 & 1 \\ 0 & -2 & 0 \\ 1 & 0 & 0 \end{array}\right).
\]
Therefore we have $\|V^* \psi(y)\|^2 = 4y^2 + y^4$ for all $y\in\R$. In particular we have that for every $c = (c_1,\dots,c_T)^\top\in\C\subseteq\R$
\[
  \sqrt{\frac{1}{T} \sum_{t=1}^T \|V^*\psi(c_t)\|_\H^2}  = \sqrt{\frac{1}{T} \sum_{t=1}^T 4c_t^2 + c_t^4}
  \]
Now we can provide the value for the constant $q_{\C,\ell,T}$ for the two case considered in this work. 

{\bf Case $\C = \mathcal B = [-B,B]^T$.} The supremum is achieved at the edges of the cube $\mathcal B$, e.g. $c = B(1,\dots,1)^\top\in\R^T$, namely
\[
  q_{\mathcal B,\ell,T} = 2 \sup_{c\in\C} ~ \sqrt{\frac{1}{T} \sum_{t=1}^T \|V^*\psi(c_t)\|_\H^2} = 2\sqrt{4B^2 + B^4} \leq 2\sqrt{5}~ B^2
\]

{\bf Case $\C = \mathcal S_{B,T}$.} the sphere of radius $B$ in $\R^T$ centered in zero. Since $\|c\| = B$ for any $c\in\mathcal S_B$ we have
\[
  q_{\mathcal S_{B,T},\ell} = 2 \sqrt{ \frac{4B^2 + B^4}{T} }
\]
Since $c_t^2 \leq B^2$ for all $c\in\mathcal S_{B,T}$ and $t=1,\dots,T$. Therefore $\sum_{t=1}^T c_t^4 \leq B^2 \sum_{t=1}^Tc_t = B^4$. However such value is achieved for $c = (B,0,\dots,0)^\top\in\R^T$, hence $q_{\C,\ell,T} = B \sqrt{4 + B^2}$. We conclude 
\[
 q(\mathcal S_{B,T},\ell) \leq 2\sqrt{5} ~ B^2 ~T^{-1/2}
\]
as desired.
\end{proof}

\section{(Most) MTL Loss Functions are SELF}
We conclude with a note on the results reported in \secref{sec:theory} characterizing sufficient conditions for the SELF property to be satisfied by either mono-variate loss functions or a separable loss sum of SELF functions. 
\PSmooth*
\begin{proof}
The result is a corollary of Thm. $19$ in \cite{ciliberto2016}. In particular $(i)$ is a direct application of Thm. $19$ point $2$, while $(ii)$ follows from the combination of Thm. $19$ point $4$, namely the fact that we are studing the separable loss on a compact subset of $\R^T$ and point $6$, implying that the sum of SELF functions is indeed SELF. 
\end{proof}

\section{Nonlinear Multitask learning and Square Loss}

We observed that for the square loss it is possible to derive a more compact and interpretable formulation for the solutions of the algorithms considered in this work (see \autoref{ex:ls-and-C} and \autoref{ex:ls-and-perturbations}). Here we show how these solution are derived in detail. In particular, in the notation of \autoref{ex:ls-and-perturbations} denote the solution $\widehat f_0:\X\to\C$ to the {\em exact} MTL problem with nonlinear constraints $\C$. $f_0$ needs to satisfy \eqref{eq:loss-trick}, where, for $\L(y,y') = \|y = y'\|^2$ the weighted sum $\sum_{i=1}^n \alpha_i(x)~\L(y,y_i)$ is
\eqals{\label{eq:ls-proj}
  \sum_{i=1}^n \alpha_i(x) \|c - y_i\|^2 = a(x) \|c\|^2 - 2 c^\top b(x) + const
}
with $a(x) = \sum_{i=1}^n \alpha_i(x)$ and $b(x) = \sum_{i=1}^n \alpha_i(x)y_i$. Therefore $f_0$ corresponds to the projection
\begin{equation}\label{eq:ls-and-C-solution}
\widehat f_0(x) = \argmin_{c\in\C}~ \left\|c - b(x)/a(x)\right\|^2 = \Pi_\C \left(b(x)/a(x)\right)
\end{equation}

\noindent{\bf Solution for the Robust $C_\delta$ in \autoref{ex:ls-and-perturbations}}. The algorithm with constraint set $\C_\delta$ takes the form 
\eqal{
    \widehat f_\delta(x) = \argmin_{c\in\C, \|r\|\leq\delta} ~ \sum_{i=1}^n \alpha_i(x) \|c + r - y_i\|^2.
}
Following the same reasoning as in \autoref{ex:ls-and-C}, we have that this problem is equivalent to 
\eqal{\label{eq:minimization-ls-slacked}
    \widehat f_\delta(x) = \argmin_{c\in\C, \|r\|\leq\delta} ~ \|c + r - b(x)/a(x)\|^2.
}
By solving \eqref{eq:minimization-ls-slacked} in $r$ we have
\eqal{
    r(c) = \argmin_{\|r\|\leq\delta} \|r - (b(x)/a(x) - c)\|^2
}
which is solved by
\eqal{
r(c) = (b(x)/a(x) - c) \min(1,\frac{\delta}{\|b(x)/a(x) - c\|}).
}
Substituting $r(c)$ in \eqref{eq:minimization-ls-slacked}, we have 
\eqal{
    c_\delta    & = \argmin_{c\in\C} \|c + r(c) - b(x)/a(x)\|^2 \\
                & = \argmin_{c\in\C} \|c - b(x)/a(x)\|^2 (1 - \min(1,\delta/\|c - b(x)/a(x)\|) )^2
}
which is minimized for $c_\delta = \Pi_\C(b(x)/a(x)) = \widehat f_0(x)$. Therefore, $\widehat f_\delta(x)$ is
\eqal{
    \widehat f_\delta(x) = c_\delta + r_\delta = \widehat f_0(x) + (b(x)/a(x) - \widehat f_0(x)) \min(1,\frac{\delta}{\|b(x)/a(x) - \widehat f_0(x)\|})
}
as stated in \autoref{ex:ls-and-perturbations}.\\

\noindent{\bf Solution for Perturbed loss $\L_\mu$ in \autoref{ex:ls-and-perturbations}}. For simplicity of notation denote $\delta = 1/\mu$. We have that the vector-valued learning algorithm with perturbed loss $\L_{1/\delta}$ is
\eqal{\label{eq:minimization-ls-perturbed}
    \widehat f_{1/\delta}(x) = \argmin_{c\in\C, r \in \R^M} ~ \sum_{i=1}^n \alpha_i(x) (\|c + r - y_i\|^2 + \delta\|r\|^2).
}
By deriving the functional w.r.t. $r$ and setting it to zero we have the equation, 
\eqal{
    \sum_{i=1}^n \alpha_i(x) ( (c + r - y_i) + \delta r) = 0 
}
which implies that the minimizer $r(c)\in\R^M$ of \eqref{eq:minimization-ls-perturbed} for any $c\in\C$ fixed, is
\eqal{
    r(c) = \frac{b(x) - a(x)c}{(1+\delta)a(x)}\in\R^M.
}
Now, plugging $r(c)$ into $\sum_{i=1}^n \alpha_i(x) \|c + r - y_i\|^2$, we have 
\eqal{\label{eq:perturbation-1}
    \sum_{i=1}^n  \alpha_i(x) & \|c + r(c) - y_i\|^2 = \sum_{i=1}^n \alpha_i(x) \| \frac{\delta c}{(1 + \delta)} - y_i + \frac{b(x)}{(1+\delta)a(x)}\|^2 \\
    & = \frac{\delta^2 a(x)}{(1 + \delta)^2}\|c\|^2 - 2 \frac{\delta }{(1+\delta)}c^\top \left(\sum_{i=1}^n \alpha_i(x) y_i - \frac{b(x)}{(1+\delta)}\right) + \textrm{const} \\
    & = \frac{\delta^2}{(1 + \delta)^2} \left(a(x)\|c\|^2 - 2 c^\top b(x) \right) + \textrm{const}
}
where we have used the fact that $b(x) = \sum_{i=1}^n \alpha_i(x)y_i$ and denoted with $\textrm{const}$ every addend not depending on $c$.

We now insert $r(c)$ in $\sum_{i=1}^n\alpha_i(x)\delta\|r\|^2$ and obtain
\eqal{\label{perturbation-1}
    \sum_{i=1}^n\alpha_i(x)\delta\|r\|^2 & = \frac{a(x)\delta}{(1 + \delta)^2a(x)^2}\|b(x) - a(x)c\|^2 \\
    & = \frac{\delta}{(1 + \delta)^2a(x)}\left(a(x)^2\|c\|^2 - 2 a(x) c^\top b(x)\right) + \textrm{const} \\
    & = \frac{\delta}{(1 + \delta)^2} (a(x)\|c\|^2 - 2 c^\top b(x)) + \textrm{const}.
}
We can therefore plug $r(c)$ into \eqref{eq:minimization-ls-perturbed}, obtaining 
\eqal{
    c_{1/\delta}    & = \argmin_{c\in\C} \sum_{i=1}^n \alpha_i(x) (\|c + r(c) - y\|^2 + \delta\|r(c)\|^2) \\
                    & = \argmin_{c\in\C} \frac{\delta}{(1 + \delta)} \left(a(x)\|c\|^2 - 2 c^\top b(x) \right) \\
}
which is minimized again for $c_{1/\delta} = \Pi_\C(b(x)/a(x)) =: \widehat f_0(x)$. By evaluating $r_{1/\delta} = r(c_{1/\delta})$ we have
\eqal{
    r_{1/\delta}(x) = \frac{b(x)/a(x) - \widehat f_0(x)}{1 + \delta}
}
and $\widehat f_{1/\delta}(x) = \widehat{f}_0(x) + r_{1/\delta}(x)$. Finally, by taking $\mu = 1/\delta$ we recover the solution at \autoref{ex:ls-and-perturbations}.

We conclude by reporting a closed form solution for the MTL estimator with perturbed loss functions $\ell_\mu^t$ for each task. The derivation is analogous to that for the vector-valued case. 

\begin{example}[Nonlinear MLT and Violations]
With the same notation as in \autoref{ex:ls-and-mtl} let us define $\bar a(x) = \sum_{t=1}^T \sum_{i=1}^{n_t} \alpha_{it}(x)$. Denote $\widehat f_0:\X\to\C$ the map such that for all $x\in\X$
\eqal{
    \widehat f_0(x) = \argmin_{c\in\C} \sum_{t=1}^T \frac{a_t(x)}{a_t(x) +  \bar a(x)/\mu} \left(c_t - \frac{b_t(x)}{a_t(x)}\right)^2
}
which corresponds to the projection of the vector $(b_1(x)/a_1(x),\dots,b_T(x)/a_T(x))$ onto $\C$ according to the diagonl matrix $\Sigma_x$ with diagonal elements $a_t(x)/(a_t(x) + \bar a(x)/\mu)$ for $t=1,\dots,T$.

The solution with perturbed loss functions $\ell_\mu^t$ is $f_\mu:\X\to\R^T$ such that for all $x\in\X$, the $t-th$ coordinate of $\widehat f_\mu(x)$ is 
\eqal{
    \widehat f_\mu (x)_t = \frac{b_t(x)/a_t(x) + \bar a(x)/\mu ~ \widehat f_0(x)}{a_t(x) + \bar a(x)/\mu}.
}
\end{example}




\end{bibunit}

\end{document}